\documentclass[sigconf,authorversion,nonacm]{acmart}

\AtBeginDocument{%
  \providecommand\BibTeX{{%
    \normalfont B\kern-0.5em{\scshape i\kern-0.25em b}\kern-0.8em\TeX}}}

\setcopyright{none}
\copyrightyear{2018}
\acmYear{2018}
\acmDOI{10.1145/1122445.1122456}

\acmConference[Woodstock '18]{Woodstock '18: ACM Symposium on Neural
  Gaze Detection}{June 03--05, 2018}{Woodstock, NY}
\acmBooktitle{Woodstock '18: ACM Symposium on Neural Gaze Detection,
  June 03--05, 2018, Woodstock, NY}
\acmPrice{15.00}
\acmISBN{978-1-4503-XXXX-X/18/06}

\usepackage{amsfonts}
\usepackage{natbib}
\usepackage{caption,subcaption,tabularx}
\usepackage{graphicx,float}
\begin{document}

\title{Multilevel Memetic Hypergraph Partitioning with Greedy Recombination}

\author{Utku Umur Acikalin}
\email{u.acikalin@etu.edu.tr}
\orcid{0000-0002-0381-8831}
\author{Bugra Caskurlu}
\orcid{0000-0002-4647-205X}
\authornotemark[1]
\email{b.caskurlu@etu.edu.tr}
\affiliation{%
  \institution{TOBB University of Economics and Technology}
  \city{Ankara}
  \country{TURKEY}
}

\renewcommand{\shortauthors}{Acikalin and Caskurlu}

\begin{abstract}
The Hypergraph Partitioning (HGP) problem is a well-studied problem that finds applications in a variety of domains. The literature on the HGP problem has heavily focused on developing fast heuristic approaches. In several application domains, such as the VLSI design and database migration planning, the quality of the solution is more of a concern than the running time of the algorithm. KaHyPar-E is the first multilevel memetic algorithm designed for the HGP problem and it returns better quality solutions, compared to the heuristic algorithms, if sufficient computation time is given. In this work, we introduce novel problem-specific recombination and mutation operators, and develop a new multilevel memetic algorithm by combining KaHyPar-E with these operators. The performance of our algorithm is compared with the state-of-the-art HGP algorithms on $150$ real-life instances taken from the benchmark datasets used in the literature. In the experiments, which would take $39,000$ hours in a single-core computer, each algorithm is given $2, 4$, and $8$ hours to compute a solution for each instance. Our algorithm outperforms all others and finds the best solutions in $112$, $115$, and $125$ instances in $2, 4$, and $8$ hours, respectively.

\end{abstract}




\maketitle

\section{Introduction}

Hypergraphs are generalizations of graphs such that each hyperedge (net) connects a subset of nodes. Hypergraphs are used to model complex relationships between nodes that cannot be captured by graphs. For instance, they can represent logic circuits containing gates with more than two inputs \cite{karypis1999multilevel}. A classical computational problem related to hypergraphs is the Hypergraph Partitioning (HGP) problem.
In this problem, we are given a Hypergraph $H$, an integer $k$, and a maximum imbalance ratio $\epsilon$, and asked to find a \textit{balanced} partition of nodes into $k$ disjoint blocks that minimizes a cost function defined over the hyperedges. A partition is called balanced if the size of each block is not more than $(1+\epsilon)$ times the average block size. The HGP problem has various applications in real life such as VLSI design \cite{karypis1999multilevel}, parallel matrix multiplication \cite{patoh}, database migration planning \cite{subramani2018minimization}, and database sharding \cite{socialhashpartitioner}. In parallel matrix multiplication, the HGP problem is used to accelerate the main computation, thus the HGP problem needs to be solved fast. For applications, such as VLSI design and database migration planning, not only that the HGP problem is solved offline, but also the quality of the solution is directly related to the cost of the operation. Ergo even small improvement in solution quality is critical \cite{wichlund1998multilevel} and translates into millions of dollars of reduction in operational costs. For instance, in application specific integrated circuit (ASIC) design, one can spend hours or days to solve the HPG problem since it will typically take weeks to create the final implementation \cite{schlag2021high}. In this paper, we target the applications of the HGP problem for which the quality of the solution is the main concern.

Most of the work on the HGP problem in the literature has focused on developing heuristic solutions since the HGP problem is NP-Hard\cite{garey1974some}. Almost all of the state-of-the-art algorithms developed for the HGP problem use the \textit{multilevel hypergraph partitioning} (or, shortly \textit{multilevel}) paradigm \cite{bulucc2016recent} that consists of the following three stages: \textit{coarsening}, \textit{initial partitioning}, and \textit{uncoarsening}. At the coarsening stage, nodes of the hypergraph are contracted at each level to obtain a series of smaller hypergraphs that are structurally similar to the original hypergraph. This stage ends when a hypergraph with the predetermined size is obtained. Then, at the initial partitioning stage, the coarsest (smallest) hypergraph is partitioned using heuristic algorithms. At each level of the uncoarsening stage, coarsening is reverted and the partition is refined using local search methods.

The efficiency of the local search algorithms decreases as the size of the hyperedges increases since the gain of moving a single node from one block to other is zero with high probability \cite{mann2014formula}. The multilevel approach allows more complex operations or more repetitions of algorithms in coarser levels without increasing the overall running time much \cite{bulucc2016recent}. Thus, the multilevel approach increases the performance of the local search algorithms since moving a single node in coarser levels actually corresponds to moving all of the nodes that contracted in that single node \cite{bulucc2016recent}. This way, local search algorithms have a broader view in coarser levels and a more granular view in the first levels \cite{andre2018memetic}.

Local search algorithms are still prone to get stuck in suboptimal solutions, and thus repeated executions of the multilevel HGP algorithms help to obtain better quality solutions. Typically, repeated executions are combined with methods that diversify the search \cite{andre2018memetic}. One commonly used approach for this is the use of \textit{V-cycles} \cite{Vcycle}. In the Multilevel HGP context, the V-cycle method uses the already computed partition and coarsens only the nodes in the same blocks. Since only the nodes in the same block are coarsened, the original partition is a valid partition, and thus the initial partitioning stage is skipped. The partition is improved at the uncoarsening stage. Using different random seeds, V-cycle method can be executed many times to improve the partition. Even though using repeated executions of V-cycles is more efficient than calculating new partitions from scratch, using more sophisticated metaheuristics is more efficient \cite{andre2018memetic}.

\citet{andre2018memetic} proposed the first multilevel memetic algorithm for the HGP problem that uses specifically tailored mutation and recombination operators. This memetic algorithm, referred to as KaHyPar-E, outperforms all other multilevel algorithms on a very large benchmark set, where each algorithm is given $8$ hours of computation time to compute a partition \cite{andre2018memetic}. Several non-multilevel (flat) evolutionary algorithms had been developed for the HGP problem before the introduction of KaHyPar-E, however, none of these algorithms is regarded as competitive as the state-of-the-art multi level HGP tools \cite{cohoon2003evolutionary}.

Memetic algorithms (MA) are population-based metaheuristics that combine local search with genetic algorithms (GA) \cite{memeticSurvey}. Genetic algorithms mimic the biological evolution process using selection, recombination (crossover) and mutation operations to evolve population of individuals throughout generations to create fit individuals \cite{holland1992genetic}. A genetic algorithm starts with an initial population of individuals each of which corresponds to a random solution \cite{burke1998initialization}. A function, called fitness function, is used to measure the fitness of an individual \cite{mccall2005genetic}. At each generation, parents are selected from the population using a selection rule that generally favors fitter individuals \cite{goldberg1991comparative}. This way, the genes of weaker individuals disappear while the genes of fitter individuals are preserved over the generations \cite{srinivas1994genetic}. The genes of the parents are mixed using recombination operators to create one or more offspring \cite{holland1992genetic}. Each offspring replaces an individual from the population. There are several commonly used evolution schemes that control the evolution of the population. On one extreme, there is the \textit{complete replacement} scheme in which the next generation is completely composed of new offspring. On the other extreme, there is the \textit{steady-state} scheme in which only one individual is replaced by a new offspring \cite{mccall2005genetic}. Another widely used evolution scheme is called \textit{replacement-with-elitism}, and in this scheme only a small proportion of the current generation is preserved \cite{mccall2005genetic}. Genetic algorithms can prematurely converge to a local minimum if the diversity of the population is not maintained \cite{baker1985adaptive}. Mutation operators increase the diversity of the population by randomly changing the genes of individuals \cite{mutation2006survey}.

Several evolutionary algorithms are devised for the HGP problem over the years. Most of these algorithms do not use the multilevel paradigm, and are outperformed by the state-of-the-art multilevel HGP algorithms \cite{andre2018memetic}. On the other hand, KaHyPar-E, which incorporates the multilevel paradigm into its operators, is the only competitive evolutionary algorithm. In contrast to other evolutionary algorithms, which use mutation and recombination operators to obtain partitions, KaHyPar-E uses operators to guide the coarsening stage of a multilevel algorithm.

Our main contribution is to develop problem-specific mutation and recombination operators for the HGP problem that effectively explores the search space given a large amount of time. We introduce one recombination and two mutation operators and combine them with the KaHyPar framework to develop a new memetic algorithm. Our memetic algorithm outperforms KaHyPar-E and three most successful multilevel HGP algorithms kKaHyPar, PaToH, and hMetis on a large benchmark set, where each algorithm is given $2,4,8$ hours of computation time to find a partition. Our algorithm finds the best solutions in $112$, $115$, and $125$ instances in $2, 4$, and $8$ hours, respectively. Our experimental study indicate that the solutions found by our algorithm in $2$ hours are better than the solutions found by kKaHyPar, PaToH, and hMetis in $8$ hours. Furthermore, the solutions found by our algorithm in $4$ hours are better than the solutions found by KaHyPar-E in $8$ hours.

The rest of the paper is organized as follows. In Section \ref{sec:related_work}, we introduce the notation used throughout the paper and formally define the HGP problem. We review the existing literature on metaheuristics devised for the HGP problem, and the KaHyPar framework along with the operators of KaHyPar-E in Section \ref{sec:preliminaries}. In Section \ref{sec:mma}, we introduce our mutation and recombination operators. We evaluate the impact of our operators and compare our algorithm with the state-of-the-art HGP algorithms with three different time limits in Section \ref{sec:exp}. We conclude in Section \ref{sec:conc}.

\section{Preliminaries}
\label{sec:related_work}

An undirected hypergraph $H$ is a four-tuple $(V,N,w,c)$, where $V$ is the set of nodes, $N \subseteq 2^{V} \setminus \emptyset$ is the set of hyperedges (nets), $w: V \rightarrow \mathbb{R}_{\ge 0}$ is the node weight function, and, $c: N \rightarrow \mathbb{R}_{\ge 0}$ is the hyperedge cost function.

The set of nodes connected by a hyperedge $e$ are called the pins of $e$ and denoted by  $pins(e)$. We denote the set of hyperedges incident to a node $v$ with $N(v) \subseteq N$. A $k$-way partition $\Pi = \{V_1, V_2, \ldots, V_k\}$ of a hypergraph $H$ is a partition of $V$ into $k$ disjoint blocks such that $\cup_{i=1}^k V_i = V$, $V_i \cap V_j = \emptyset$ if $i \not= j$, and each block $V_i$ is not empty. A $k$-way partition $\Pi$ of $H$ is called an $\epsilon$-balanced $k$-way partition if the size of each block $V_i$ is no more than the average block size, i.e., for each $i$, $\sum_{v \in V_i} w(v) \le (1+\epsilon) \lceil \dfrac{\sum_{v \in V} w(v)}{k}\rceil$.

The set of blocks that contains a pin of a hyperedge $e$ is called the connectivity set of $e$, and denoted by $\Lambda(e)$. The size of the connectivity set of a hyperedge $e$ is called the connectivity of $e$ and denoted by $\lambda(e)$. 
Hyperedges with connectivity more than one are called cut hyperedges. We use $N_\Pi$ to denote the set of cut hyperedges in $\Pi$.

The $k$-way hypergraph partitioning problem is to find an $\epsilon$-balanced $k$-way partition $\Pi$ of a given hypergraph $H$ that minimizes a cost function defined over the cut nets $N_\Pi$. The two most commonly used cost functions are the \textit{cut} and \textit{connectivity} metrics. The cut metric is defined as the total cost of cut hyperedges, i.e.,  $cut_H(\Pi) = \sum_{e \in N_\Pi} c(e)$. The connectivity metric, also referred to as $(\lambda -1)$,  takes into account how many block that each hyperedges spans, and defined as $(\lambda - 1)_H(\Pi) = \sum_{e\in N} (\lambda(e) - 1)\ c(e)$. The special case of the problem where $k=2$ is referred to as bipartition or bisection. Two metrics are identical when $k = 2$, and they are both NP-hard to optimize\cite{garey1974some}. We focus on the connectivity metric in this paper since this is the most commonly used metric in the literature.

Node contraction merges two nodes $u,v$ into a single node $u$. After the contraction, weight of $u$ increases as much as the weight of $v$, i.e. $w(u) = w(u) + w(v)$. All hyperedges in $N(v) \setminus N(u)$ are updated by replacing $v$ with $u$, and $v$ is removed from all hyperedges in $N(v) \cap N(u)$. Node uncontraction reverts all these operations.

\section{Related Work and KaHyPar Framework}
\label{sec:preliminaries}
The HGP problem received lots of attention from VLSI and scientific computing communities due to its applications in these domains. Empirical results show that  general purpose KaHyPar \cite{akhremtsev2017engineering,heuer2017improving,schlag2016k,kahyparMF, kahyparHFC}, VLSI focused hMetis \cite{karypis1999multilevel,karypis2000multilevel}, and scientific computing focused PaToH \cite{patoh} stand out among the multilevel HGP algorithms from the perspective of the quality of the solutions found\cite{andre2018memetic}. The existing literature is too vast to summarize here, so we refer readers to \cite{alpert1995recent,schlagDis} for an extensive overview. Here we focus on evolutionary approaches developed for the HGP problem, and the KaHyPar framework which we incorporate our algorithm in.

\citet{saab1989evolution} proposes an evolutionary algorithm for solving a more complex HGP problem with a multi-objective cost function. The algorithm has only one individual generated using First Fit Decreasing heuristic \cite{johnson1973near}, and evolves this solution randomly by changing the blocks of nodes if the value of the objective function impoves more than an integer taken as parameter.
\citet{hulin1990circuit} proposes a genetic algorithm that uses a complex two-step encoding scheme for the Circuit (Hypergraph) Bipartitioning problem (i.e., $k = 2$). First, the (complex) components of the circuit are grouped, then these groups are divided and mapped to the blocks. They develop crossover and mutation operators suited for their encoding-scheme. The initial population is randomly generated by first selecting a random group for each component and then randomly dividing groups. \citet{bui1994fast} present a steady-state memetic algorithm for Hypergraph Bipartitioning problem with the ratio-cut metric. They use a 5-point crossover operator, a basic mutation operator that changes the blocks of nodes randomly, and a weak variation of the FM algorithm \cite{fiduccia1982linear} for the local search. They use a preprocessing step that reindexes the nodes by the visiting order of depth first search on the clique representation \cite{hagen1992new} of the hypergraph. They claim this improves the performance of the crossover operator because nodes that will likely to be assigned together have closer indexes. \citet{areibi2000integrated} proposes a memetic algorithm that uses a variant of $k$-way FM algorithm \cite{sanchis1989multiple} for Circuit Partitioning with the graph model. This algorithm is adapted for hypergraph partitioning and improved by using a preprocessing step that contracts nodes to reduce the complexity of the problem \cite{areibi2004effective}. Also, they use a solution found by a GRASP metaheuristic along with random solutions to create the initial population.

The first algorithm that combines multilevel paradigm with evolutionary algorithms is proposed by \cite{soper2004combined} for the Graph Partitioning problem. Their recombination and mutation operators modify the edge weights of the input graph such that the underlying multilevel partitioner is guided while searching for a new partition. \citet{benlic2011multilevel} presents a multilevel memetic algorithm for the Graph Partitioning problem where $\epsilon = 0$. They use a multi-parent recombination operator which ensures that the balance of the parent solution is not degraded and refine the offspring with a perturbation-based Tabu Search algorithm. \citet{sanders2012distributed} proposes a parallel multilevel evolutionary algorithm for the Graph Partitioning problem. They use two crossover operators that ensure the solution quality does not degrade. The first crossover operator restricts the contraction of nodes that are cut edges in at least one of the parents. The second crossover operator uses Natural Cuts \cite{delling2011graph}, which is originally proposed as a preprocessing method for partitioning of road networks. They also use two mutation operators based on V-cycles.

More recently, \citet{andre2018memetic} presented the first multilevel memetic algorithm, called KaHyPar-E, for the HGP problem. They use recombination and mutation operators that either restrict the contraction of some nodes, or modify the contraction scores of nodes. We describe KaHyPar-E in more detail in Section \ref{sec:kahypar_memetic}. \citet{preen2019evolutionary} proposes an evolutionary algorithm for the initial partitioning stage. They use a uniform crossover operator with parental alignment, and a mutation operator that randomly changes the blocks of nodes. They compare their algorithm with the initial partitioning algorithms used in kKaHyPar after applying FM algorithm to both initial partitions. Solution quality does not differ much when the hypergraph is coarsened until there are $150 \cdot k$ nodes (this is a common value in multilevel algorithms), but they obtain better solutions when there are more than $15000 \cdot k$ nodes.

\subsection{KaHyPar Framework}

KaHyPar is a general purpose hypergraph partitioning framework developed over the series of papers \cite{schlag2016k,heuer2017improving,andre2018memetic,kahyparMF,kahyparHFC,schlag2021high}. The latest version of the k-way direct partitioner in KaHyPaR framework, referred to as kKaHyPar, is considered the best partitioner \cite{schlag2021high}. We first present a high level description of the kKaHyPar algorithm since both our algorithm and KaHyPar-E use it as a subroutine.

kKaHyPar employs two preprocessing techniques before the coarsening stage. The first one is called pin sparsification and it contracts some pins of large hyperedges to reduce the average pin size of hyperedges. The second method uses Louvain algorithm \cite{blondel2008fast} to detect the community structure of the hypergraph \cite{heuer2017improving}. This information is used to guide the coarsening process.

At each level of the coarsening stage, kKaHyPar contracts two nodes into a single node. Notice that the coarsening stage is composed of nearly $n$-levels. To rate node pairs for contraction, it uses the heavy-edge function \cite{karypis1999multilevel}. For a node pair $(u,v)$, the heavy-edge function is defined as follows.

$$ r(u,v) = \sum_{e \in N(u) \cap N(v)} \dfrac{c(e)}{|pins(e)| -1}$$

At each level of the coarsening stage, kKaHyPar chooses a random node $v$ and than finds its contraction partner $v$ from the \textit{eligible nodes} for which $r(u,v)$ is maximized. Nodes $u$ and $v$ are eligible for contraction if they belong to the same community (found at the preprocessing step), and $w(u) + w(v) \le \kappa$, where $\kappa = \lceil \frac{\sum_{v \in V} w(v)}{t \cdot k} \rceil$ for some parameter $t$. In kKaHyPar, $t$ is selected as $150$ \cite{kahyparHFC}. The coarsening stage ends when there are less than $t \cdot k$ nodes, or there are no valid contractions, i.e., no node has an eligible partner.

To find the initial partition to the coarsest hypergraph, kKaHyPar recursively bipartitions the hypergraph until a $k$-way partition is found \cite{schlag2016k}. It employs a pool of $9$ randomized algorithms to compute bipartitions \cite{schlag2021high}, and runs each algorithm $20$ times. Each partition is refined using local search algorithms based on FM heuristics \cite{fiduccia1982linear, sanchis1989multiple}, and the best partition is used as the initial partition. Then this initial partition is projected to larger hypergraphs at the uncoarsening stage. kKaHyPar uses $k$-way FM heuristic along with a local search algorithm based on maximum flows (MF) \cite{kahyparMF,kahyparHFC}. $k$-way FM heuristic is executed a large number of times at each level of the uncoarsening stage. The more complex and time-consuming MF heuristic is executed once at each level $l$ where $l$ is an exact power of two. For details on the MF heuristic, see \cite{kahyparMF, kahyparHFC}.

\subsection{KaHyPar-E}
\label{sec:kahypar_memetic}
KaHyPar-E is a multilevel memetic algorithm that outperforms all other hypergraph partitioners when given a large amount of time\cite{andre2018memetic}. KaHyPar-E takes a time limit $t$ as an input. The initial population is generated using kKaHyPar algorithm. Since kKaHyPar also takes non-negligible time, the size of the population is determined dynamically as follows. First, kKaHyPar is executed once and its running time is recorded. Then, roughly $15\%$ of the time limit $t$ is used to generate the initial population. The population size is bounded below with $3$ to ensure diversity, and it is bounded above with $50$ to ensure convergence.
This initial population is then evolved over generations with the steady-state paradigm where only one offspring is created at each generation via a mutation or a recombination operation. KaHyPar-E uses two mutation operators and two recombination operators, and each operator is picked with a probability of $0.25$. The individual that will be replaced by the offspring is chosen by a replacement strategy that considers both solution quality (fitness) and the similarity of individuals. The fitness of an individual is determined by its objective function value, and the one with a lower objection function value is considered fitter. Contrary to other evolutionary algorithms developed for the HGP problem, it uses problem-specific mutation and combination operators instead of problem-agnostic operators. We next describe these operators as well as the replacement strategy. We use the terms individual and solution interchangeably. 

\medskip

\textbf{Recombination Operators:} KaHyPar-E uses two recombination operators. The first one, which we refer to as $C_1$, works similar to the V-cycles but it is more restrictive. First, it selects two parents using two-way tournaments. The fitter parent becomes the first parent and the other becomes the second parent. Let $b_i[u]$ denote the block of node $u$ in parent $i$. This operator coarsens two nodes $u$ and $v$ if and only if $b_1[u] = b_1[v] \land b_2[u] = b_2[v]$, i.e., if the nodes $u$ and $v$ are placed to the same blocks in both parents. This constraint ensures that the partitions of both parents are still valid partitions after any possible contraction. Contracting nodes beyond the maximum allowed node weight $\kappa$ could lead to unbalanced initial partitions. But the contraction restriction ensures that the partitions of each parent are still valid after coarsening. Therefore, instead of finding new initial partition, partition of the fitter parent is used. Then, the solution is refined at the uncoarsening stage. This operator actually creates a new coarsening scheme for the first parent and then applies local search algorithms at the uncoarsening stage. Thus, it ensures that the solution quality is maintained even if no improvements are found. But, this also means that the offspring created will be the same as the first parent if no improvements are found. Normally, this could lead to a premature convergence since the fittest individual could possibly replace all other individuals if no improvements were found for the fittest individual throughout the generations. But, the similarity measure used in the replacement scheme ensures that no individual is allowed to be in the population more than once.

The second recombination operator, which we refer to as $C_2$, uses the fittest $p$ individuals, for some parameter $p$, to determine the nodes to be coarsened. The \textit{frequency} of a hyperedge $e$, denoted by $f(e)$, is defined as the number of the fittest $p$ individuals for which $e$ is in the cut. Hyperedge frequencies are incorporated into the rating function used at the coarsening stage by replacing the heavy-edge function with the following function, which favors the node pairs that share a large number of small hyperedges with low frequencies.

$$ r(u,v) = \dfrac{1}{w(u) \cdot w(v)} \sum_{e \in N(u) \cap N(v)}\dfrac{exp(-\gamma \cdot f(e))}{|pins(e)|} $$

In this function, $\gamma$ is a tuning parameter that controls the impact of frequencies, which is chosen as $0.5$. The number of individuals $p$ is chosen as the square root of the population size. 

\medskip

\textbf{Mutation Operators:} KaHyPar-E uses two mutation operators, both of which is based on V-cycles. The first mutation operator, which we refer to as $M_1$, is exactly the same as the V-cycle method. It creates a new coarsening scheme for a solution by blocking the contractions of the nodes that are placed in different blocks. After the coarsening stage, $M_1$ uses the original partition as the initial partition, and proceeds with the uncoarsening stage. The coarsening stage of the second mutation operator, which we refer to as $M_2$, is the same as that of $M_1$. As opposed to $M_1$, $M_2$ does not skip the initial partitioning stage, but proceeds with the initial partitioning stage of kKaHyPar to find an initial partition. This initial partition is refined at the uncoarsening stage. Notice that $M_1$ ensures that the solution quality will not degrade. However, as is the case with $C_1$, the solution will not change if no improvements are found. This is not the case for the second mutation operator.

\medskip

\textbf{Replacement Strategy:} Recall that $C_1$ and $M_1$ finds an existing solution if no improvement is found at the uncoarsening stage. Therefore, maintaining the diversity of the population is very important for KaHyPar-E. Evicting the worst solution from the population may lead to a premature convergence. Thus, a replacement strategy that considers the \textit{similarity} between individuals as well as the solution quality is used. It uses a sophisticated similarity measure instead of the Hamming distance, which is used in HGP algorithms \cite{bui1994fast,kim2004hybrid}. In this similarity measure, each individual $I_i$ is represented with a multiset $D_i$ in which each cut hyperedge $e$ of individual $I_i$ appears $(\lambda(e) - 1)$ times. The difference between two individuals $I_i$ and $I_j$ is defined as the cardinality of the symmetric difference between $D_i$ and $D_j$, i.e., $d(I_1,I_2) = |(D_1 \setminus D_2) \cup (D_2 \setminus D_1)|$. When an offspring $o$ is created, it replaces the individual whose difference to $o$ is smallest, among the pool of individuals whose fitness is no more than that of $o$.

\section{Multilevel Memetic Algorithm}
\label{sec:mma}
Our algorithm is built on top of the KaHyPar-E algorithm. Unless stated otherwise, we use the components of the KaHyPar-E as described above. We first introduce a novel greedy recombination operator that contracts nodes using the blocks of two parents. Contrary to the recombination operator $C_1$, which only contracts nodes that both parents agree on, our greedy recombination operator contracts a subset of nodes in the same block of one of the parents. The block that will be used for contraction is determined by using a rating function. We then introduce two new mutation operators that slightly differ from $M_1$ and $M_2$. Proposition \ref{prop:contraction} given below presents the theoretical insight that led to our greedy recombination operator.

\begin{proposition}
\label{prop:contraction}
Let $\Pi =\{V_1,V_2,\ldots, V_k\}$ be an optimal $\epsilon$-balanced k-way partition of a hypergraph $H$. Let $u$ and $v$ be two nodes that are in the same block in $\Pi$. Let $H'$ be the hypergraph obtained from $H$ by contracting $u$ and $v$. The cost of an optimal $\epsilon$-balanced k-way partition of $H'$ is equal to the cost of $\Pi$.	
\end{proposition}

\begin{proof}
Consider the partition $\Pi'$ that is obtained by contracting nodes $u$ and $v$ in $\Pi$. Notice that $\Pi'$ is an $\epsilon$-balanced k-way partition of $H'$, and the cost of $\Pi'$ is equal to that of $\Pi$. Thus, the cost of an optimal $\epsilon$-balanced k-way partition of $H'$ is no more than that of $\Pi$. Let $\Pi^*$ be an optimal $\epsilon$-balanced k-way partition of $H'$. All we need is to show that the cost of $\Pi^*$ is no less than that of $\Pi$.

For the sake of contradiction, assume that the cost of $\Pi^*$ is strictly less than that of $\Pi$. Let $\tilde{\Pi}^*$ be the partition obtained by contracting nodes $u$ and $v$ in $\Pi^*$. But then, $\tilde{\Pi}^*$ is an $\epsilon$-balanced k-way partition of $H$, and the cost of $\tilde{\Pi}^*$ is equal to that of $\Pi^*$. This contradicts with $\Pi$ being an optimal $\epsilon$-balanced k-way partition of $H$.  	

\end{proof}

\subsection{Greedy Recombination Operator}

Let $\Pi$ be an optimal $\epsilon$-balanced k-way partition of a hypergraph $H$. 
Let $H'$ be the hypergraph obtained from $H$ by contracting a subset (possibly all) of nodes of any block of $\Pi$. Notice that, due to the proof of Proposition \ref{prop:contraction},  the problem of finding an optimal  $\epsilon$-balanced k-way partition of a hypergraph $H$ is equivalent to to the problem of finding that of the smaller (contracted) hypergraph $H'$. However, deciding if any subset of nodes $V'$ appears in the same block of an optimal partition of a hypergraph is NP-hard. Our greedy recombination operator is motivated by the observation of Proposition \ref{prop:contraction}, and aims to approximate the aforementioned NP-hard problem.

The greedy recombination operator selects two parents by two-way tournaments. Then it sorts the blocks of each parent with respect to a \textit{quality measure}. It selects the block with the highest quality, and assigns all nodes in this block to a new cluster. It removes these nodes from the blocks of the other parent, and recomputes the qualities of these blocks. This procedure is repeated until either each node is assigned to a cluster, or $\frac{3k}{2}$ blocks are selected. The nodes that are not assigned to a cluster are assigned to a special cluster called $0$. These cluster are used to guide the coarsening stage of the kKaHyPar algorithm as follows. At the coarsening stage, two nodes $u$ and $v$ are only allowed to be contracted if they belong to the same cluster, and this cluster is not cluster $0$. The coarsening stage ends only if there are no possible contractions. This operator uses kKaHyPar algorithm with this modified coarsening stage to create an offspring.

The motivation for selecting no more than $\frac{3k}{2}$ blocks is as follows. Clustering the nodes in the remaining $\frac{k}{2}$ blocks may lead to less ideal contractions since we already selected the best $\frac{3k}{2}$ block. In addition to that, allowing contractions even after there are less than $t \cdot k$ nodes remaining may create small number of nodes with uneven weights, which may complicate the initial partitioning stage \cite{karypis2000multilevel}.

The rating function used to measure the block quality is crucial for the performance of the greedy recombination operator. Different operators can be developed by using different rating functions. We believe that a rating function needs to satisfy the following conditions.

\begin{itemize}
	\item The quality of a block should not depend on the number of nodes in contains since we remove the contracted nodes from the other parent.
	\item It should rank a \textit{good} block higher compared to a \textit{bad} block.
\end{itemize}

We say that a block is  good if it contains either a high or a low fraction of pins of incident hyperedges. For instance, consider three blocks $V_1$, $V_2$, and $V_3$, each of which is incident on $2$ hyperedges. $V_1$ contains $\frac{1}{4}$ and $\frac{3}{4}$ of the pins of the incident hyperedges. $V_2$ contains $\frac{1}{2}$ of the pins of both of the incident  hyperedges. $V_3$ contains $\frac{1}{2}$ and $\frac{3}{4}$ of the pins of incident hyperedges. Notice that $V_1$ is defined to be a better block than $V_2$. This is because $V_1$ is easier to improve using local search algorithms. Furthermore, we want $V_3$ to be better than $V_1$ since it has higher average of fraction of pins. We use the rating function given below which satisfies these conditions. Let $E_i = \cup_{v \in V_i}N(v)$.

$$ r(V_i) = \frac{1}{|E_i|}\sum_{e \in E_i} \left(\dfrac{|pins(e) \cap V_i|}{|pins(e)|}\right)^2 $$

Notice that our greedy recombination operator is not particularly useful when there are only two blocks. Since for $k=2$ the greedy recombination operator finds one of the parent solutions. We use $C_3$ to denote the greedy recombination operator.

\subsection{Mutation Operators}
The mutation operators cluster the nodes of the hypergraph using an already computed partition. They use these clusters to create a new coarsening scheme by blocking some contractions. Operators first find the connected components in each block. If a connected component does not contain all pins of at least one hyperedge, then nodes on this connected component are assigned to the special cluster $0$. Then, all of the remaining connected components on each block $i$ are assigned to a single cluster number $i$.

The first operator only allows the contraction of nodes $u$ and $v$ if $u$ and $v$ are assigned to the same cluster, and the assigned cluster is not the special cluster $0$. Notice that this operator does not allow the contraction of nodes $u$ and $v$ that are in different blocks. Therefore, after the modified coarsening stage, the original partition is valid and can be used as an initial partition to the coarsened hypergraph. Thus, the initial partitioning stage is skipped and the partition is refined at the uncoarsening stage. Notice that if each block only contains a single connected component, and there is at least one hyperedge whose pins are assigned to this block, then this operator works the same as the mutation operator $M_1$.

The second operator allows the contraction of nodes $u$ and $v$ only if $u$ and $v$ are assigned to the same cluster, or either $u$ or $v$ is assigned to the special cluster $0$. The original partition may not be feasible after the modified coarsening stage since nodes from different blocks can be contracted. The modified coarsening stage of this operator is followed by the initial partitioning stage of kKaHyPar. Then, the initial partition is refined at the uncoarsening stage. Notice that if each block only contains a single connected component, and there is at least one hyperedge whose pins are assigned to this block, then this operator works the same as the mutation operator $M_2$. We use $M_3$ and $M_4$ to denote these mutation operators, respectively.

\section{Experimental Evaluation}
\label{sec:exp}
\textbf{Experimental Setup:} We implemented $C_3$, $M_3$ and $M_4$ using KaHyPar framework. 
We performed all experiments on a cluster with $3$ machines each of which has 2 Intel Xeon 6148 Icosa-core processors clocked at 2.4GHz, and 384 GB RAM. We solved instances in parallel and one core and 8GB is reserved for each instance. We compare our algorithm with KaHyPar-E, kKaHyPar, PaToH 3.2 \cite{patoh} and hMetis-R \cite{karypis1999multilevel, karypis2000multilevel}. While our algorithm and KaHyPar-E spend all of their time evolving a population, kKaHyPar, PaToH, and hMetis-R compute new partitions using different random seeds until the time limit is reached. We use kKaHyPar with its default parameters given in \cite{schlag2021high} with the V-cycle refinement technique at most $100$ times after computing a partition as in \cite{andre2018memetic}. hMetis-R uses a different balance definition and do not directly optimizes connectivity metric. Instead, it optimizes Sum of External Degrees (SOED) metric which is closely related to the connectivity metric \cite{schlag2021high}. We used PaToH and hMetis-R with their default parameters, and chose $\epsilon$ for hMetis-R as described in \cite{andre2018memetic}.

We evaluate the performance of different configurations of our algorithm in Section \ref{sec:impact}. For that, each configuration of our algorithm is run $3$ times with different random seeds and, with a $2$ hour time limit. To save time we only used the instances with $32$ blocks. We used the same random seeds for each configuration so that they start with the same initial population. We use \textit{convergence plots} to show how the mean connectivity of best solution of each instance is evolved over time for each configuration.

We compare the performance of our algorithm with KaHyPar-E in Section \ref{sec:compare_kahypar_e}, and with the other state-of-the-art hypergraph partitioners in Section \ref{sec:compare_all}. We run each algorithm $5$ times with different seeds for each instance with $2,4$ and $8$ hour time limits, respectively. We run our algorithm and KaHyPar-E with the same seeds so that they start with the same initial population to mitigate the impact of randomness. We use \textit{performance plots} \cite{andre2018memetic} to compare the performances of algorithms with the best found solution on a per-instance basis.

\textbf{Convergence plots}: Convergence plots show how the mean connectivity of best solutions for each instance evolved over time. First, we compute the best solution found until time $t$ for each instance and configuration. To combine the results obtained by different seeds for a single instance, we use arithmetic mean. To combine the results of different instances, we use geometric mean so that each instance can contribute equally. The y-axis shows the mean connectivity of all instances used and the x-axis shows the time.

\textbf{Performance plots}: 
For each instance, the connectivity of the solution found by each algorithm is divided to that of the the best solution for that instance. This is repeated for each instance and results are aggregated such that the plot shows the fraction of instances where each algorithm produces solutions that are $r$ times worse compared to the best. The x-axis shows the quality relative to the best solution and the y-axis shows the fraction of instances. For instance, a point $(r,f)$ for an algorithm $A$ in a performance plot shows that the algorithm $A$ finds solutions that are at most $r$ times higher than the best solution in $f$ fraction of the instances. Notice $r$ is monotone non-decreasing. 
At $r=1$, performance plots show the fraction of instances where each algorithm found the best solution.

\begin{table}[h]
	\caption{Benchmark Set D\cite{andre2017evolutionary}}
	\label{tab:hypergraph_stats}
	
	\begin{tabular}{lccc}
		Hypergraph & n & m & p \\ \hline
		\multicolumn{4}{c}{ISPD89} \\ \hline
		ibm06 & 32,498 & 34,826 & 128,182 \\
		ibm07 & 45,926 & 48,117 & 175,639 \\
		ibm08 & 51,309 & 50,513 & 204,890 \\
		ibm09 & 53,395 & 60,902 & 222,088 \\
		ibm010 & 69,429 & 75,196 & 297,567 \\ \hline
		\multicolumn{4}{c}{SAT14 Dual}  \\ \hline
		6s133 & 140,968 & 48,215 & 328,924 \\
		6s153 & 245,440 & 85646 & 572692 \\
		6s9   & 100,384 & 34317 & 234228 \\
		dated-10-11-u & 629,461 & 141,860 & 1,429,872 \\
		dated-10-17-u & 1,070,757 & 229,544 & 2,471,122 \\ \hline
		\multicolumn{4}{c}{SAT14 Literal}  \\ \hline
		6s133 & 96,430  &  140,968 &  328,924 \\
		6s153 &  171,292 &  245,440  & 572,692 \\
		aaai10-planning-ipc5  &  107,838  & 308,235 &  690,466 \\
		\verb|atco_enc2_opt1_05_21|  & 112,732  & 526,872  & 2,097,393 \\
		dated-10-11-u &  283,720  & 629,461  & 1,429,872 \\ \hline
		\multicolumn{4}{c}{SAT14 Primal}  \\ \hline		
		6s153 & 85,646  &  245,440 &  572,692 \\
		aaai10-planning-ipc5  &  53,919  & 308,235 &  690,466 \\
		\verb|atco_enc2_opt1_05_21|  & 56,533  & 526,872  & 2,097,393 \\
		dated-10-11-u &  141,860  & 629,461  & 1,429,872 \\
		hwmcc10-timeframe &163,622 & 488,120 & 1,138,944 \\ \hline
		\multicolumn{4}{c}{SPM} \\ \hline
		\verb|laminar_duct3D| & 67,173 &67,173 &3,833,077 \\
		\verb|mixtank_new| & 29,957 & 29,957 & 1,995,041 \\
		\verb|mult_dcop_01| & 25,187  & 25,187  &193,276 \\
		RFdevice & 74,104 & 74,104 & 365,580 \\
		vibrobox &12,328 & 12,328 & 342,828 \\ \hline
	\end{tabular}
\end{table}

\textbf{Instances}: We use Benchmark Set D of \citet{andre2018memetic} with $25$ unweighted hypergraphs that consists of the instances from ISPD98 VLSI Circuit Benchmark Suite \cite{alpert1998ispd98}, the SuiteSparse Matrix Collection \cite{davis2011university}, and the SAT Competition 2014 \cite{belov2014proceedings}. The respective sizes and application domains of the hypergraph are presented in Table \ref{tab:hypergraph_stats}. Each hypergraph is partitioned into $k$ blocks where $k \in \{4,8,16,32,64,128\}$, and we set $\epsilon$ to $0.03$ as in \cite{andre2018memetic}. Thus, we have $150$ instances in total. Our experiments take $39,000$ CPU hours in a single-core computer.

\subsection{Impact of Algorithmic Components}
\label{sec:impact}
The details of the configurations we considered are given in Table \ref{tab:confiration}. The first column lists the configuration names, and the last $2$ columns give the selection probabilities of recombination and mutation operations for each configuration Columns $2-4$ show the selection probability of each recombination operator when the recombination operation is selected. Columns $5-8$ show the selection probability of each mutation operator when the mutation operation is selected.
\begin{table}[h]
	\caption{Configuration Table}
	\label{tab:confiration}
	\begin{tabularx}{\columnwidth}{l|XXX|XXXX|XX}
		\hline
		Configuration & \multicolumn{7}{c|}{Operator Selection Probability}\\
		Name & ${C_1}$ & $C_2$ & $C_3$ & $M_1$ & $M_2$ & $M_3$ & $M_4$ & $C$ & $M$  \\ \hline
		KaHyPar-E & 0.5 & 0.5 & 0 & 0.5 & 0.5 & 0 & 0 & 0.5 & 0.5\\
		MMA-M-0.5 & 0.4 & 0.2 & 0.4 & 0.25 & 0.25 & 0.25 & 0.25 & 0.5 & 0.5\\
		MMA & 0.4 & 0.2 & 0.4 & 0.25 & 0.25 & 0.25 & 0.25 & 0.8 & 0.2\\
		MMA-G & 0 & 0 & 1 & 0.25 & 0.25 & 0.25 & 0.25 & 0.8 & 0.2\\	
		MMA-EQ-C & 0.33 & 0.33 & 0.33 & 0.25 & 0.25 & 0.25 & 0.25 & 0.8 & 0.2\\ \hline
	\end{tabularx}
	
\end{table}

Figure \ref{fig:effects} shows how the mean best solution of each configuration is evolved over time. The MMA-G configuration which uses a single recombination operator $(C_3)$ performs the worst. The MMA-M-05 which uses all recombination and mutation operators performs slightly better than KaHyPar-E which shows the effectiveness of the new operators. The MMA uses less mutation operation compared to MMA-M-05 and performs better. This shows that $C_3$ contributes to the diversity of the population since lowering mutation probability decreases the performance of KaHyPar-E \cite{andre2018memetic}. The performance of MMA-EQ-C, which uses all recombination operators with equal probability, is similar to that of MMA. We choose MMA configuration as our final configuration since it performs best.

\begin{figure}[h]
	\centering
	
	\includegraphics[width=\linewidth]{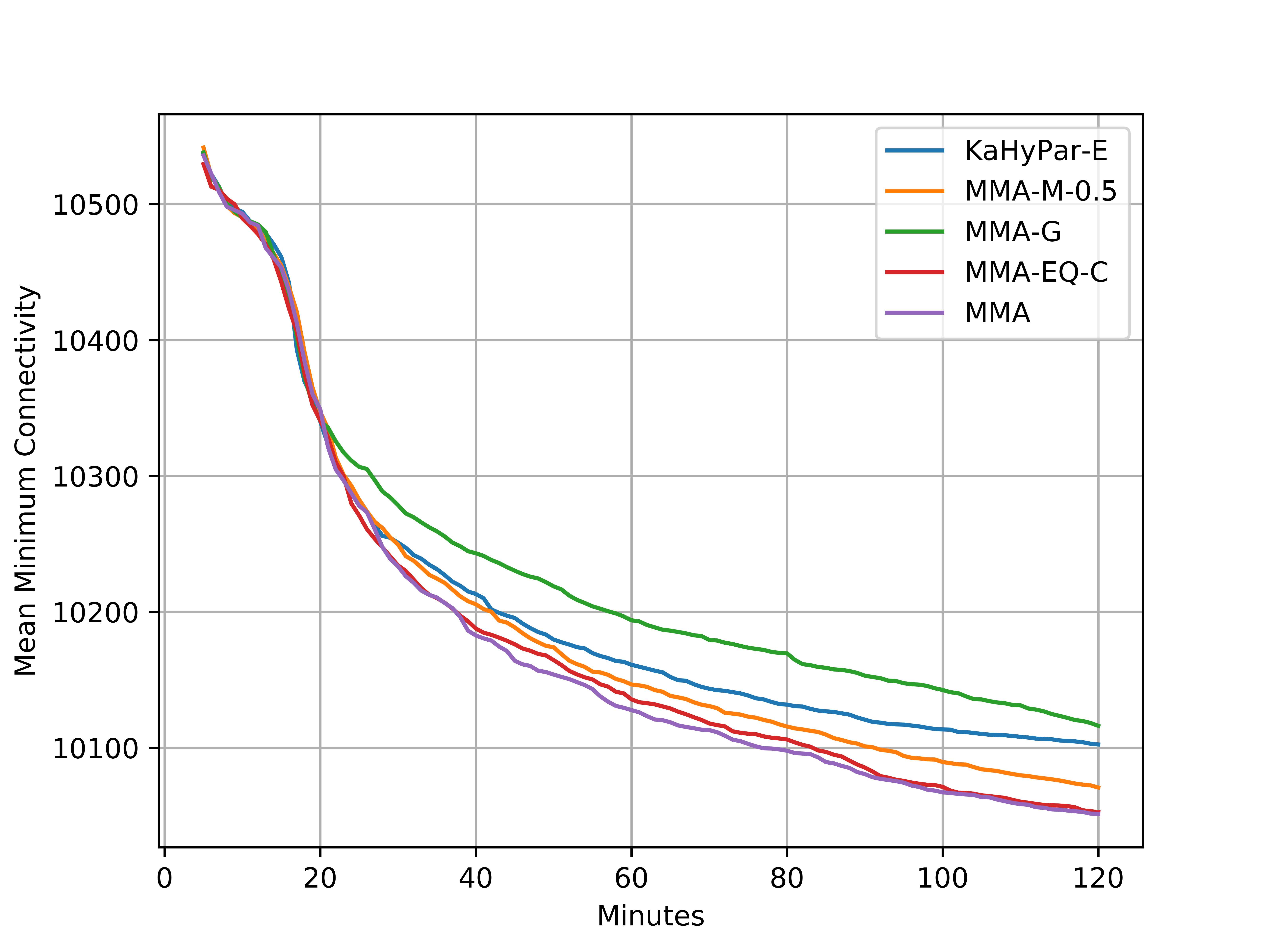}
	\caption{Convergence plot for different algorithmic configurations}
	\label{fig:effects}
\end{figure}

\subsection{Comparison with KaHyPar-E}
\label{sec:compare_kahypar_e}

\begin{figure}[h]	
	\centering

\includegraphics[width=0.95\linewidth]{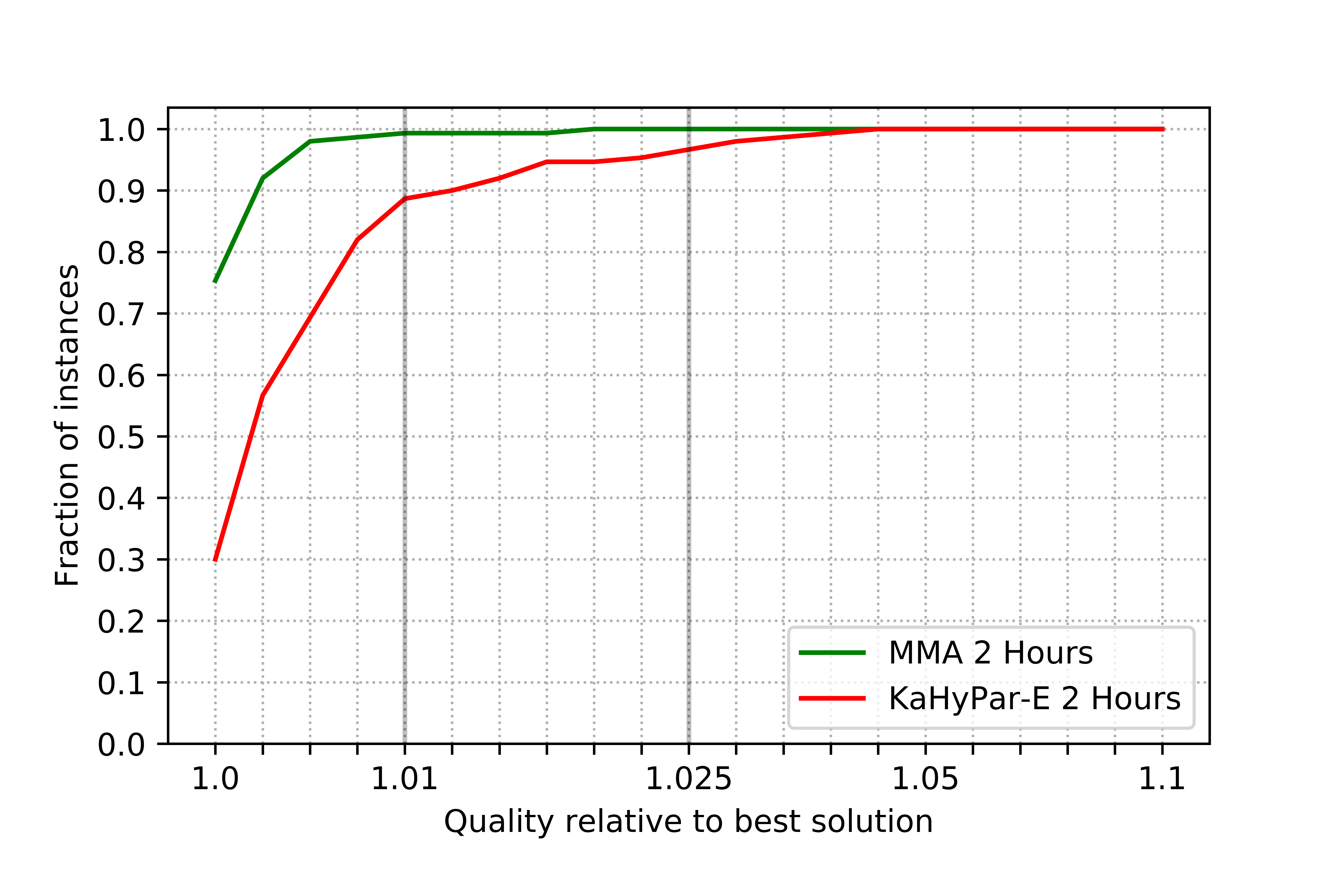}
\includegraphics[width=0.95\linewidth]{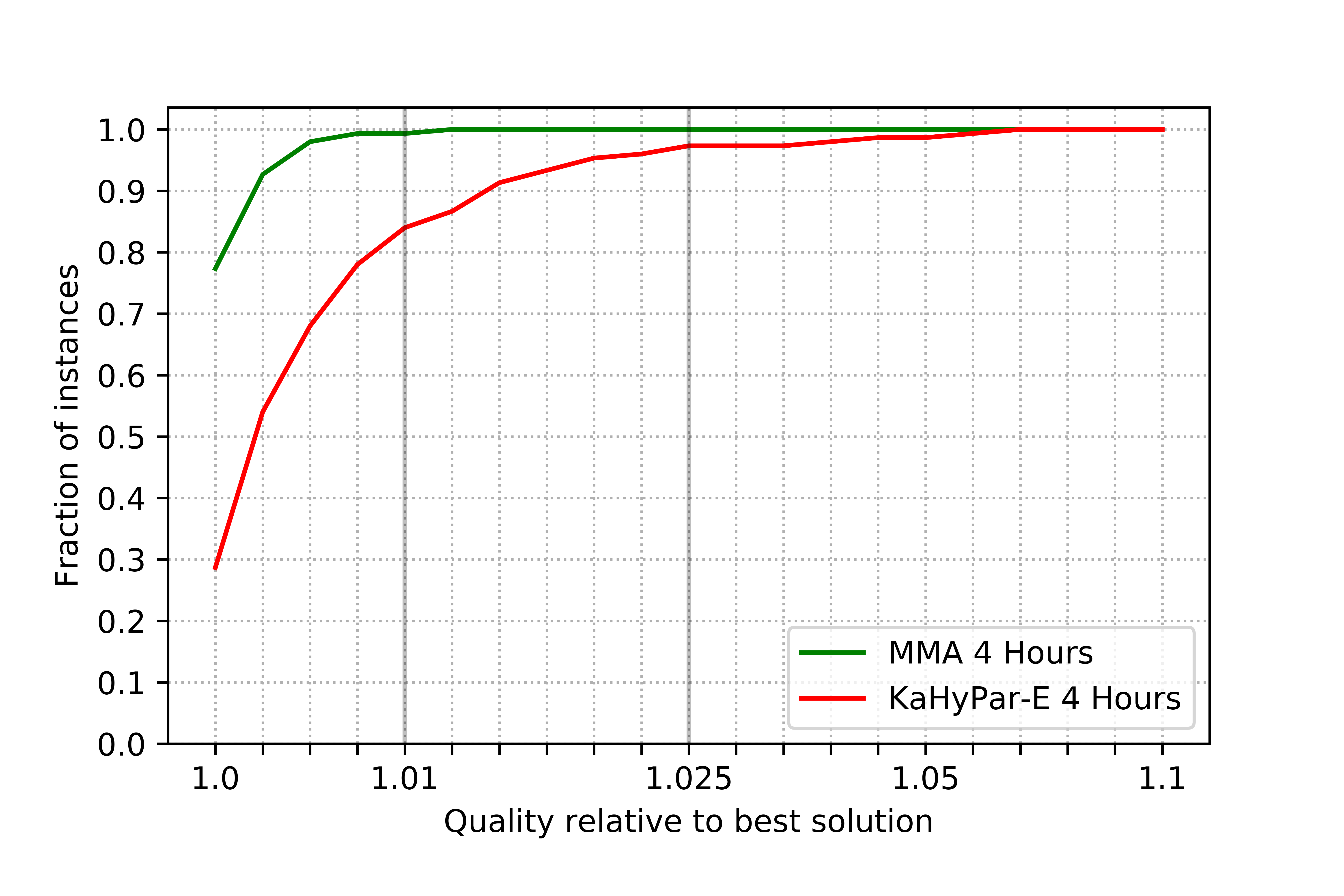}
\includegraphics[width=0.95\linewidth]{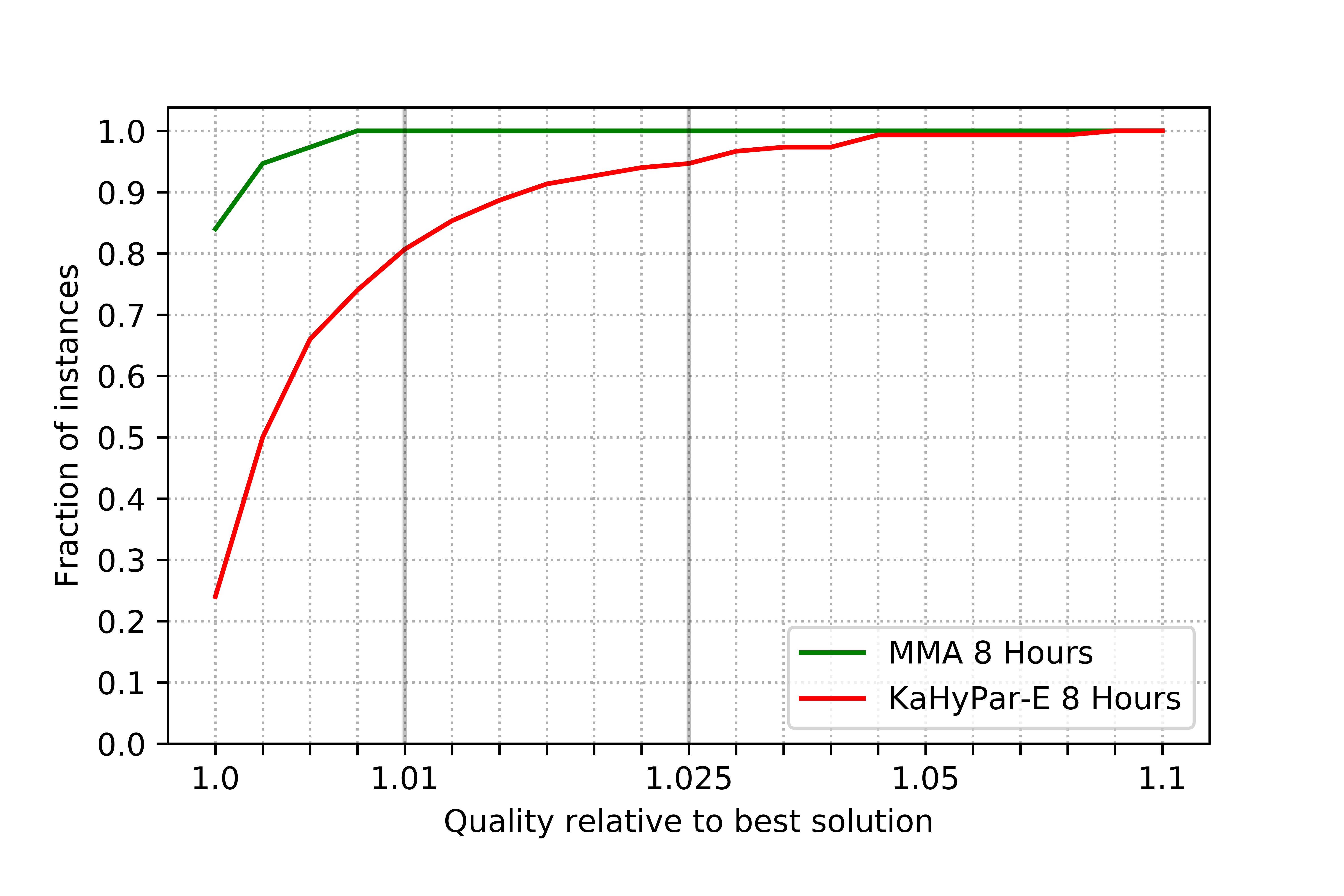}
	
	\caption{Performance plots comparing MMA and KaHyPar-E for 2, 4, and 8 hour time limits}
	\label{fig:perf_kahypare}
	
\end{figure}

Figure \ref{fig:perf_kahypare} shows that MMA consistently outperforms KaHyPar-E for each time limit. Furthermore, the performance gap widens as the computation time increases. When each algorithm is given $2$ hours, MMA computes solutions with better, equal, and worse  quality compared to KaHyPar-E on $105,8$ and $37$ instances, respectively. For the  $8$ hour time limit, MMA computes solutions with better, equal, and worse  quality compared to KaHyPar-E on $114,12$ and $24$ instances, respectively. For the setting where MMA is given $4$ hours and KaHyPar-E is given $8$ hours, MMA computes solutions with better, equal, and worse  quality on $75,11,64$ instances, respectively.

Table \ref{tab:average_improvement_kahypare} shows the average improvements over KaHyPar-E for different time limits and number of blocks. For instances with $k\ge 32$, MMA performs $0.46\%$, $0.57\%$, $0.8\%$  better on average than KaHyPar-E, in $2,4,8$ hours, respectively. For some instances, MMA returns $8.1\%$ better solutions than KaHyPar-E, however, there is no instance for which KaHyPar-E returns a solution that is $1\%$ better than MMA. A Wilcoxon matched pairs signed rank test \cite{wilcoxon1992individual} show that the average improvements by MMA over KaHyPar-E are statistically significant for all time limits (all $p$ values are less than $10^{-10}$).

\begin{table}[h]
	\caption{Left-side of the table shows the average improvement over KaHyPar-E in percentages. Right-side of the table shows the number of instances that MMA found the best solutions.}
	\label{tab:average_improvement_kahypare}
	\begin{tabular}{cccc|ccc}
		\hline
		&\multicolumn{3}{c}{Average} & 	\multicolumn{3}{|c}{Best Solutions} \\
		k & 2 Hour & 4 Hour & 8 Hour & 2 Hour & 4 Hour & 8 Hour   \\
		\hline
		
		4 		& -0.11  & 0.04   & -0.06 &	16	& 18		& 18	\\
		8 		&  0.28  & 0.46   &	0.38  &	19	& 20		& 18	\\ \
		16 		&  0.73  & 0.71   &	0.73  &	23	& 20		& 25	\\
		32 		&  0.56  & 0.71   &	0.89  &	20	& 22		& 22	\\
		64  	&  0.54  & 0.54   &	0.86  &	19	& 21		& 23	\\
		128     &  0.29  & 0.46   &	0.63  &	16	& 16		& 15	\\  \hline
		all     &  0.37  & 0.49   &	0.57  &	113	& 116		& 126	\\
		\hline
		
	\end{tabular}
\end{table}

\subsection{Comparison with other heuristics }
\label{sec:compare_all}

\begin{figure}[h]
	\centering

\includegraphics[width=0.95\linewidth]{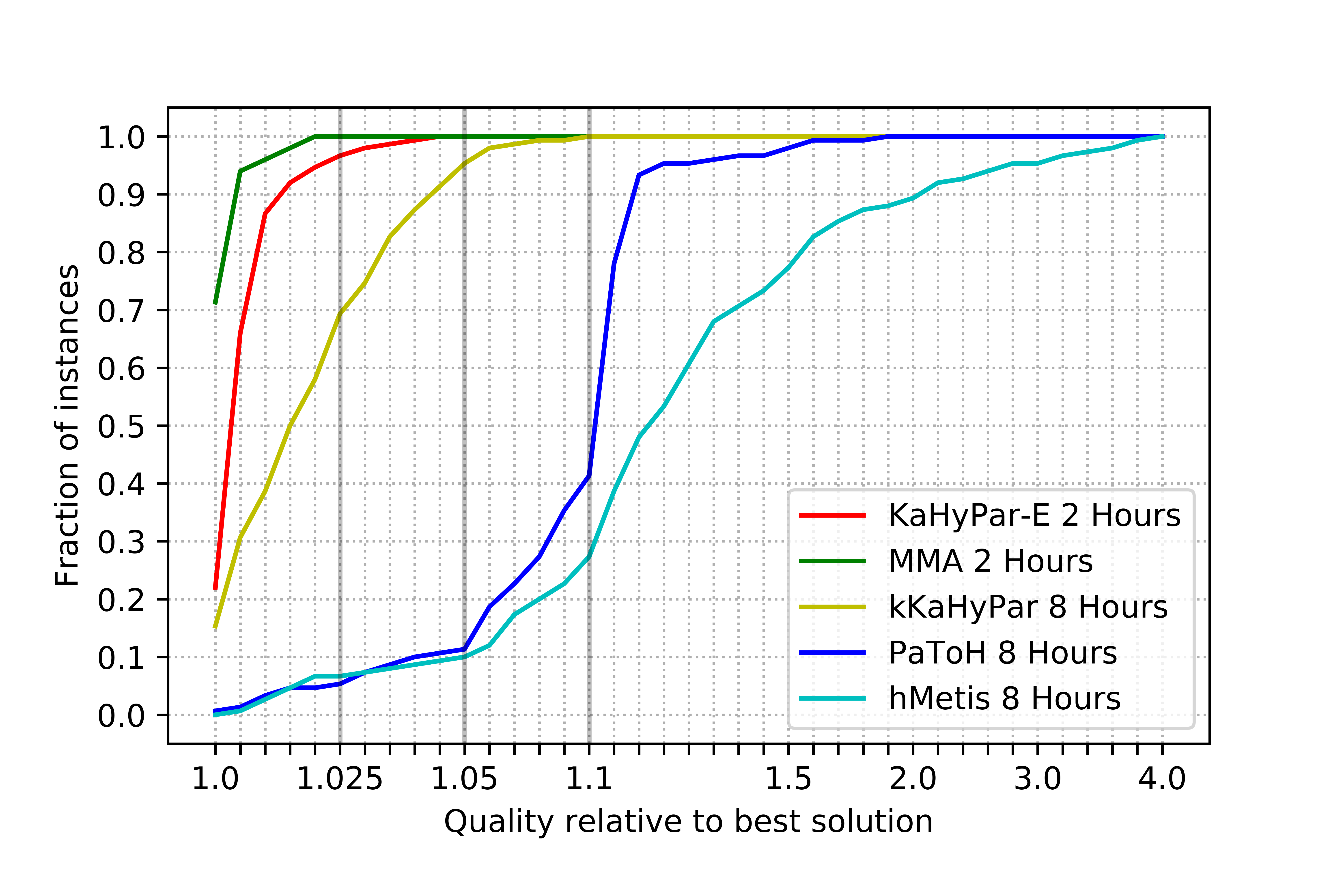}
\includegraphics[width=0.95\linewidth]{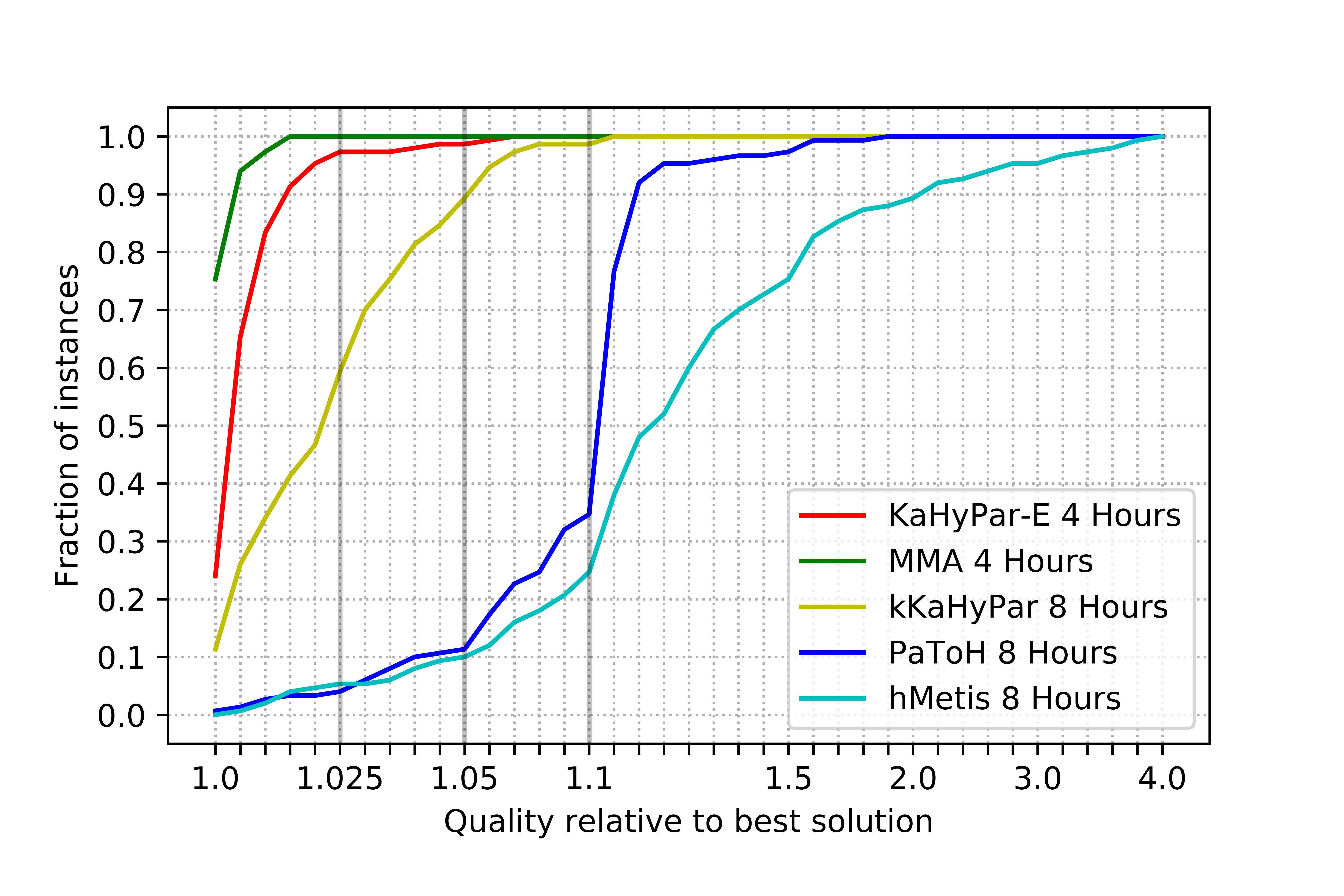}
\includegraphics[width=0.95\linewidth]{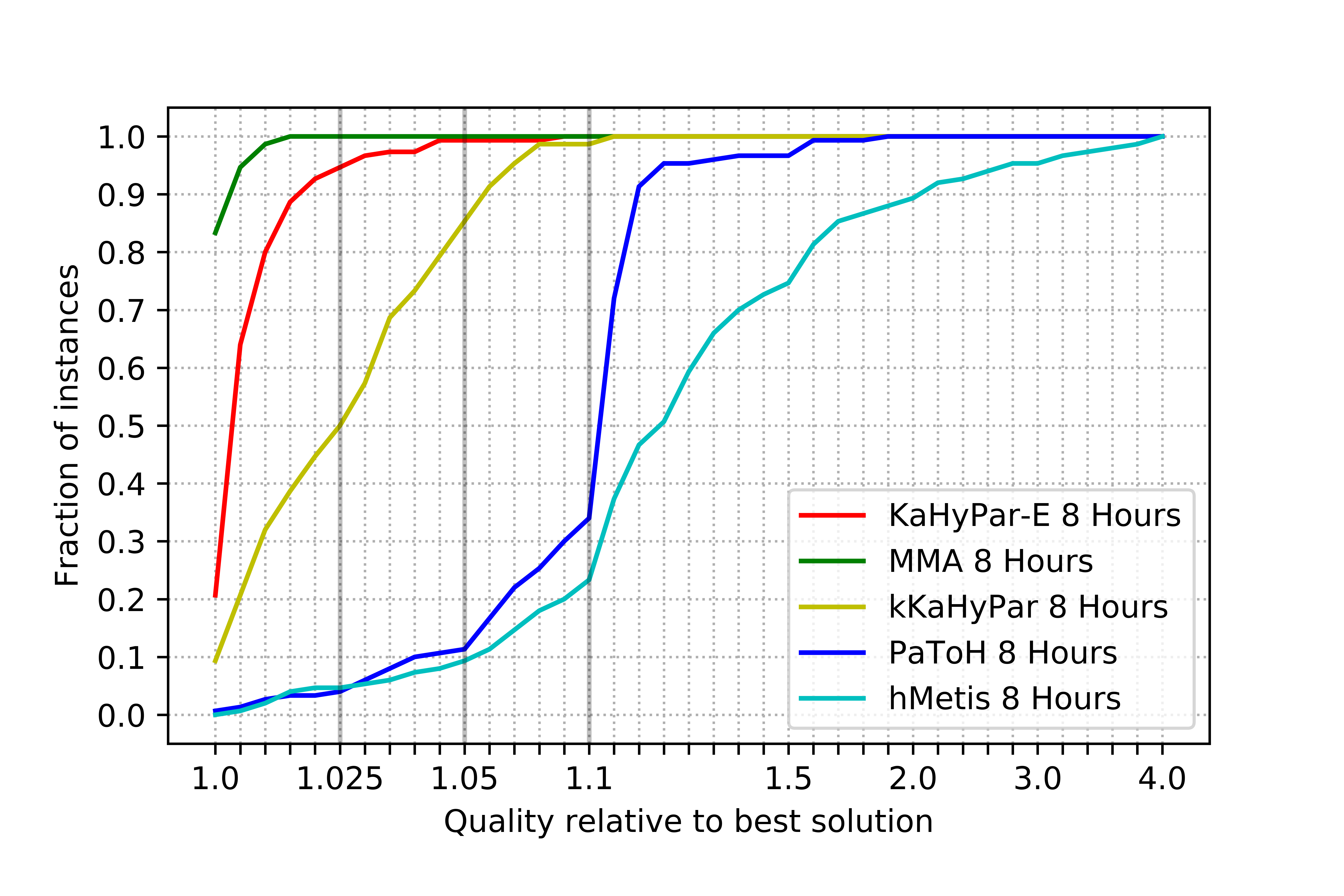}

	\caption{Performance plots comparing all algorithms}
	\label{fig:comp_others}
\end{figure}

Figure \ref{fig:comp_others} shows performance comparisons of MMA, KaHyPar-E, and $3$ other state-of-the-art multilevel hypergraph partitioners kKaHyPar, PaToH, and hMetis-R. In all plots, kKaHyPar, PaToH, and hMetis-R are run with $8$ hour time limit. MMA and KaHyPar-E are run with $2,4$, and $8$ hour time limits, respectively. MMA and KaHyPar-E outperform all non-evolutionary algorithms even when they use a quarter of the time non-evolutionary algorithms use. MMA finds a best solution in $107,113,125$ instances, and KaHyPar-E finds a best solution in $33,36,31$ instances for $2,4,8$ hours, respectively. kKaHyPar finds a best solution in $23,17,14$ instances $2,4,8$ hours, respectively. While PaToH finds a best solution for one instance, hMetis-R does not find a best solution in any of the instances. For each algorithm, the maximum performance gap with the best solution in any instance is as follows: $1\%$ for MMA, $8\%$ for KaHyPar-E, $12\%$ for KaHyPar, $82\%$ for PaToH, $284\%$ for hMetis-R. The results show that MMA is the best choice for applications where the solution quality is the most important concern.

\section{Conclusion}
\label{sec:conc}
We introduced a problem-specific greedy recombination operator and two mutation operators for the HGP problem that can be incorporated into a multilevel memetic algorithm. The greedy recombination operator combines the good genes of each parent in a novel way using a rating function that can rank subsets of nodes. We add our operators to the best performing hypergraph partitioning algorithm, KaHyPar-E, to develop a even better performing algorithm. We evaluated the performance of our algorithm in a large benchmark set and showed that our algorithm outperforms $4$ state-of-the-art algorithms. Our algorithm finds the best solution in $125$ of the $150$ instances when each algorithm is given $8$ hours. Even though the average improvement over KaHyPar-E is not large, it is statistically significant. We also showed that our algorithm with $4$ hour time limit outperforms KaHyPar-E with $8$ hour time limit. This can be partly explained by the fact that improvements for both algorithms are slow.

\clearpage
\bibliographystyle{ACM-Reference-Format}
\bibliography{references}


\end{document}